\documentclass[final]{hld2026} 
\usepackage{hyperref}
\usepackage{amsmath}
\usepackage{amssymb}
\usepackage{bm}
\usepackage{cancel}
\usepackage{mathtools}
\usepackage[toc,page]{appendix}
\usepackage{listings}
\usepackage{subcaption}
\usepackage{tcolorbox}
\usepackage[table,svgnames]{xcolor}
\usepackage{multirow}
\usepackage{hhline}
\usepackage{minitoc}
\usepackage{enumitem}
\usepackage{wrapfig}
\renewcommand \thepart{}
\renewcommand \partname{}

\DeclareCaptionFont{figfont}{\fontsize{10pt}{12pt}\selectfont}
\definecolor{darkred}{rgb}{0.55, 0.0, 0.0}
\definecolor{darkblue}{rgb}{0.0, 0.0, 0.55}

\definecolor{codegreen}{rgb}{0,0.6,0}
\definecolor{codegray}{rgb}{0.5,0.5,0.5}
\definecolor{codepurple}{rgb}{0.58,0,0.82}
\definecolor{backcolour}{rgb}{0.95,0.95,0.92}
\lstdefinestyle{mystyle}{
    backgroundcolor=\color{backcolour}, 
    commentstyle=\color{codegreen},
    keywordstyle=\color{magenta},
    numberstyle=\tiny\color{codegray},
    stringstyle=\color{codepurple},
    basicstyle=\ttfamily\scriptsize,
    breakatwhitespace=false, 
    breaklines=true,
    captionpos=b, 
    keepspaces=true,
    showspaces=false, 
    showstringspaces=false,
    showtabs=false, 
    tabsize=4
}
\renewcommand{\appendixname}{Appendix}

\usepackage{etoolbox}
\makeatletter
\patchcmd{\hyper@makecurrent}{%
    \ifx\Hy@param\Hy@chapterstring
        \let\Hy@param\Hy@chapapp
    \fi
}{%
    \iftoggle{inappendix}{
        \@checkappendixparam{chapter}%
        \@checkappendixparam{section}%
        \@checkappendixparam{subsection}%
        \@checkappendixparam{subsubsection}%
        \@checkappendixparam{paragraph}%
        \@checkappendixparam{subparagraph}%
    }{}%
}{}{\errmessage{failed to patch}}

\newcommand*{\@checkappendixparam}[1]{%
    \def\@checkappendixparamtmp{#1}%
    \ifx\Hy@param\@checkappendixparamtmp
        \let\Hy@param\Hy@appendixstring
    \fi
}
\makeatletter

\newtoggle{inappendix}
\togglefalse{inappendix}

\apptocmd{\appendix}{\toggletrue{inappendix}}{}{\errmessage{failed to patch}}

\usepackage{environ}
\usepackage{xstring}
\NewEnviron{salign}{%
    \expandarg\IfSubStr{\detokenize\expandafter{\BODY}}{\detokenize{\\}}{%
        \begin{subequations}
        \begin{align}
            \BODY
        \end{align}
        \end{subequations}
    }{%
        \begin{align}
            \BODY
        \end{align}
    }%
}

\usepackage[nameinlink]{cleveref}
\crefname{section}{Section}{Sections}
\crefname{equation}{Equation}{Equations}
\creflabelformat{equation}{#2#1#3}
\crefname{figure}{Figure}{Figures}
\crefname{lemma}{Lemma}{Lemmas}
\crefname{corollary}{Corollary}{Corollaries}
\crefname{theorem}{Theorem}{Theorems}
\crefname{assumption}{Assumption}{Assumptions}

\newcommand{\eg}{e.g.\;}

\newcommand{\ie}{i.e.\;}

\newcommand{\wrt}{w.r.t.\;}

\newcommand{\bff}{\mathbf{f}}

\newcommand{\bq}{\mathbf{q}}

\newcommand{\bu}{\mathbf{u}}
\newcommand{\bv}{\mathbf{v}}

\newcommand{\bx}{\mathbf{x}}
\newcommand{\by}{\mathbf{y}}

\newcommand{\bJ}{\mathbf{J}}

\newcommand{\bP}{\mathbf{P}}

\newcommand{\bS}{\mathbf{S}}

\newcommand{\bU}{\mathbf{U}}
\newcommand{\bV}{\mathbf{V}}
\newcommand{\bW}{\mathbf{W}}
\newcommand{\bX}{\mathbf{X}}
\newcommand{\bY}{\mathbf{Y}}

\newcommand{\cL}{\mathcal{L}}

\newcommand{\cN}{\mathcal{N}}

\newcommand{\cU}{\mathcal{U}}

\newcommand{\lr}[3]{\left#1#2\right#3}
\newcommand{\rb}[1]{\lr{(}{#1}{)}}
\renewcommand{\sb}[1]{\lr{[}{#1}{]}}
\newcommand{\cb}[1]{\lr{\{}{#1}{\}}}

\newcommand{\norm}[1]{\lr{\|}{#1}{\|}}

\newcommand{\real}{\mathbb{R}}

\newcommand{\identity}{\mathbf{I}}

\newcommand{\trace}{\text{Tr}}

\newcommand{\loss}{\cL}

\newcommand{\transpose}{^\intercal}
\newcommand{\inv}{^{-1}}
\newcommand{\pinv}{^{\dagger}}

\newcommand{\bSigma}{\mathbf{\Sigma}}
\newcommand{\bTheta}{\mathbf{\Theta}}
\usepackage{parskip}
\newcommand{\reconsider}[1]{} 

\title{How the Hessian-Spectrum of Neural Networks Depends on Data}
\hldauthor{
    \Name{Jasraj Singh}\thanks{Work done during an internship at the ELLIS Institute, T\"{u}bingen, and MPI-IS.} \Email{jasraj.singh00150@gmail.com}
    \AND
    \Name{Enea Monzio Compagnoni} \Email{eneamonziocompagnoni@gmail.com} \\
    \addr{Universit\"{a}t Basel, Basel, Switzerland}
    \AND
    \Name{Antonio Orvieto} \Email{antonio@tue.ellis.eu} \\
    \addr{ELLIS Institute, T\"{u}bingen, Germany \\
    MPI-IS, T\"{u}bingen, Germany}
}
\date{\today}

\begin{document}
\maketitle

\begin{abstract}
    The Hessian matrix is an important quantity of interest when it comes to studying the loss landscape and optimization dynamics in deep learning, as well as designing measures of generalization, second-order learning algorithms, etc. Prior works have focused on empirical results or pursued a theoretical treatment under overly simplified settings. In this work, we derive the eigenvalues of the Hessian of linear networks with arbitrary widths and depths, and datasets with an arbitrary number of samples, features, and labels. Importantly, for classification tasks with MSE loss, we identify that the sharpness of the solution is directly related to the maximum proportion of samples belonging to any class. We empirically validate our predictions and systematically analyze the effects of shedding the impractical assumptions one at a time, as well as incorporating nonlinearities. We observe that our predictions are considerably robust in most cases, allowing us to extend our conclusions to more practical learning setups.
\end{abstract}


\section{Introduction}

Neural networks induce highly nonconvex loss landscapes that are complicated to study. Despite this complexity, the Hessian matrix of the loss, encoding only its second-order information, has played a central role in designing efficient optimization algorithms \cite{martens10hessian-free,gupta18shampoo}, deriving measures of generalization \cite{chaudhari2017entropysgd,jastrzebski2018factorsinfluencingminimasgd} -- which, in turn, inspired the design of well-generalizing optimizers \cite{foret2021sharpnessaware,kwon21basam} -- as well as practical algorithms for pruning \cite{lecun89brain,hassibi93brain}, quantization \cite{frantar22brain}, continual learning \cite{mirzadeh20regimes,mirzadeh2021linear}, etc.

Our motivation resembles that of \cite{singh2024closed,singh2026cracking}, in that we aim to precisely characterize the full spectrum of the Hessian. However, while these works prioritize minimal assumptions at the cost of model complexity, 
we leverage standard assumptions in deep learning theory in exchange for networks with \emph{arbitrary widths} and \emph{depths}, and datasets with \emph{arbitrary number of samples}, \emph{features} and \emph{labels}. 
Our methodology falls closest to \cite{ghosh25matrixfat}, which studies the spectrum of the Hessian and training dynamics at the edge-of-stability (EoS) \cite{cohen2021gradient} for matrix factorization problems. In contrast, we consider the problem of supervised learning, with the overarching goal of understanding the influence of data-geometry on the sharpness of the learnt solution. 

\section{Setup}

Consider the inputs $\bx_1, \ldots, \bx_n \in \real^{1\times d_0}$ collected in $\bX \in \real^{n\times d_0}$, and outputs $\by_1, \ldots, \by_n \in \real^{1\times d_L}$ collected in $\bY \in \real^{n\times d_L}$. Our model is an $L$-layer linear neural network, defined as
\begin{salign}
    \bff\rb{\bx; \bW_1,\bW_2,\ldots,\bW_L} &= \bx \bW_{1}\transpose \bW_{2}\transpose \ldots \bW_{L}\transpose
\end{salign}
with $\bW_{\ell} \in \real^{d_{\ell} \times d_{\ell-1}}$ is the learnable weight matrix in the $\ell^{\textsuperscript{th}}$ layer, $\forall \ell \in \sb{L} \coloneqq \{1, 2, \ldots, L\}$.\footnote{The parameterization follows the implementation of a linear layer in the PyTorch framework \citep{torch2019adam}.} In what follows, we denote the model parameters by $\bW_{1:L} \coloneqq \rb{\bW_1, \bW_2, \ldots, \bW_L}$. The loss function is the Mean-Squared Error (MSE):
\begin{salign}
    \loss\rb{\bW_{1:L}; \bX,\bY} \coloneqq \frac{1}{2n} \norm{\bff\rb{\bX; \bW_{1:L}} - \bY}_F^2
\end{salign}
where $\norm{\cdot}_F$ denotes the Frobenius norm of the argument.


\subsection{Approximating the Hessian-spectrum}

We use the generalized Gauss-Newton (GGN) approximation of the Hessian, which 
gets more and more precise as loss reduces with training:
\begin{salign}
    \frac{\partial^2}{\partial \bW_{1:L}^2} \loss\rb{\bW_{1:L}; \bX,\bY} \approx \frac{1}{n} \rb{\frac{\partial \bff\rb{\bX}}{\partial \bW_{1:L}}}\transpose \rb{\frac{\partial\bff\rb{\bX}}{\partial \bW_{1:L}}} = \frac{1}{n} \bJ\transpose \bJ \in \real^{p\times p}
\end{salign}
with the total number of parameters given by $p = \sum_{\ell=1}^L d_{\ell} d_{\ell-1}$, and the Jacobian \wrt predictions given by $\bJ \coloneqq \partial \bff\rb{\bX} / \partial \bW_{1:L} = \begin{bmatrix} \bJ_1 & \bJ_2 & \ldots & \bJ_{L} \end{bmatrix} \in \real^{d_Ln\times p}$, where
\begin{salign}
    \bJ_{\ell} &= \rb{\underbrace{\bW_{\ell+1}\transpose \bW_{\ell+2}\transpose \ldots \bW_{L}\transpose}_{\bS_{\ell} \in \real^{d_{\ell}\times d_{L}}}}\transpose
    \otimes
    \rb{\bX \underbrace{\bW_{1}\transpose \bW_{2}\transpose \ldots \bW_{\ell-1}\transpose}_{\bP_{\ell} \in \real^{d_{0}\times d_{\ell-1}}}}
\end{salign}
where $\otimes$ is the Kronecker outer-product. The GGN matrix 
shares its non-zero eigenvalues with the 
Neural Tangent Kernel (NTK) matrix \cite{jacot2018ntk}, 
which is simpler to study:
\begin{salign}
    \frac{1}{n} \bJ \bJ\transpose
    &= \sum_{\ell=1}^{L}
    \rb{\bS_{\ell}\transpose \bS_{\ell}} \otimes \rb{\frac{1}{n} \bX \bP_{\ell} \bP_{\ell}\transpose \bX\transpose} \in \real^{d_Ln\times d_Ln}
\label{eqn:column-gramian}
\end{salign}

\reconsider{The eigenvalues of each summand are given by the pairwise-products:
\begin{salign}
    \lambda_{ij} \rb{\rb{\frac{1}{n} \bX\transpose \bP_{\ell}\transpose \bP_{\ell} \bX} \otimes \rb{\bS_{\ell} \bS_{\ell}\transpose}} = \sigma^2_{i}\rb{\frac{1}{\sqrt{n}} \bP_{\ell} \bX} \sigma^2_j\rb{\bS_{\ell}}
\end{salign}
where $\sigma_{i}\rb{\cdot}$ denotes the $i\textsuperscript{th}$ largest non-zero singular value of the argument.}

Going forward, we will make use of the Singular Value Decomposition (SVD) of the feature matrix, $\bX = \bU_{\bX}\bSigma\bV_{\bX}\transpose$, where $\bU_{\bX}\in\real^{n\times r}$ is a semi-orthogonal matrix, $\bSigma\in\real^{r\times r}$ is a diagonal matrix with $\bSigma_{11} \geq \bSigma_{22} \geq \ldots \geq \bSigma_{rr} > 0$, and $\bV_{\bX}\in\real^{r\times r}$ is an orthogonal matrix; $r$ is the rank of $\bX$.

\begin{assumption}
\label{ass:whitened-inputs}
    We have an overdetermined system, \ie $n\geq d_0$, the feature matrix is full-rank, \ie $r = d_0$, and the features are isotropic, \ie $\bSigma = \bSigma_{11}\identity_{d_0}$.
\end{assumption}
This is a rather strong assumption, and not realistic in practice, but it helps draw insights from the bounds we derive on the eigenvalues of the NTK.


\reconsider{Under \cref{ass:whitened-inputs}, the non-zero singular values of $\frac{1}{\sqrt{n}} \bP_{\ell} \bX$ are proportional to those of $\bP_{\ell}$:
\begin{salign}
    \sigma_{i}\rb{\frac{1}{\sqrt{n}} \bP_{\ell} \bX} = \sigma_{\bx} \cdot \sigma_{i}\rb{\bP_{\ell}}
\end{salign}}

\subsection{Shallow Networks}

\begin{theorem}
\label{thm:shallow-networks}
    The non-zero eigenvalues of shallow linear networks ($L=2$) are characterized by the pairwise-sums of squared singular values of the two layers:
    \begin{salign}
        \lambda_{ij}\rb{\frac{1}{n}\bJ\bJ\transpose} &\leq \frac{\norm{\bX}_2^2}{n} \rb{\sigma_{i}^2\rb{\bW_1} + \sigma_{j}^2\rb{\bW_2}}
    \label{eqn:shallow-nets-eigvals}
    \end{salign}
    where $i\in\sb{\min\cb{d_0,d_1}}$ and $j\in\sb{\min\cb{d_1,d_2}}$, and $\sigma_i\rb{\bW}$ denotes the $i\textsuperscript{th}$ largest singular value of $\bW$. Furthermore, the bound is attained under \cref{ass:whitened-inputs}. 
\end{theorem}

\begin{figure}[t]
\centering
    \includegraphics[width=0.325\linewidth]{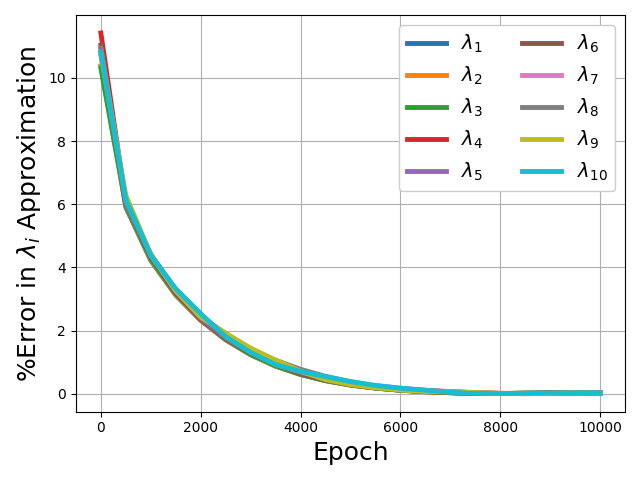}
    \includegraphics[width=0.325\linewidth]{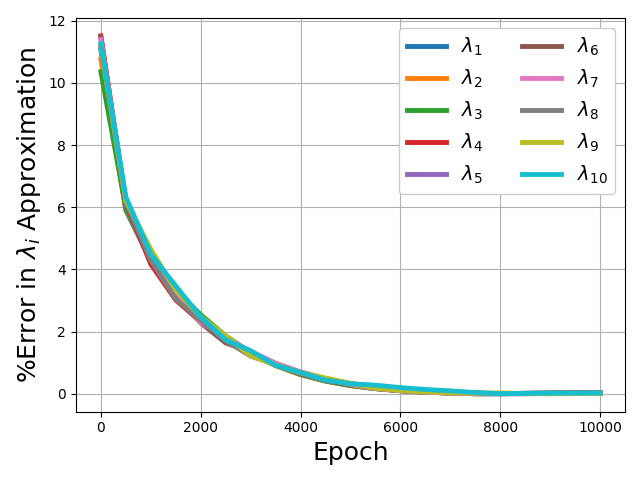}
    \includegraphics[width=0.325\linewidth]{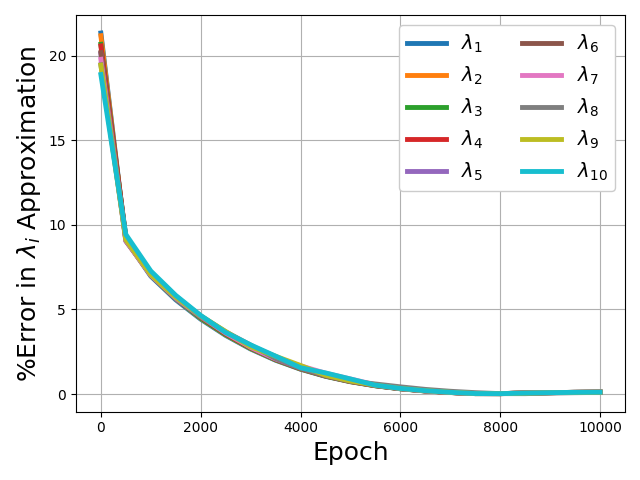}
    \caption{Percentage error in approximating top-10 eigenvalues of the Hessian (using \cref{thm:shallow-networks}) for a 2-layer MLP on (a) MNIST, (b) FashionMNIST, and (c) CIFAR10.}
\label{fig:shallow_networks}
\end{figure}

We visualize the case of isotropic features in \cref{fig:shallow_networks}, where we note that the approximation becomes increasingly precise as training progresses. \cref{thm:shallow-networks} implies that, of the $d_0d_1+d_1d_2$ eigenvalues of the Hessian, at most $\min\cb{d_0,d_1} \times \min\cb{d_1,d_2}$ may be non-zero; we comment on this at the end of \cref{sec:strongly-balanced}.

\begin{figure}[t]
\centering
    \includegraphics[width=0.325\linewidth]{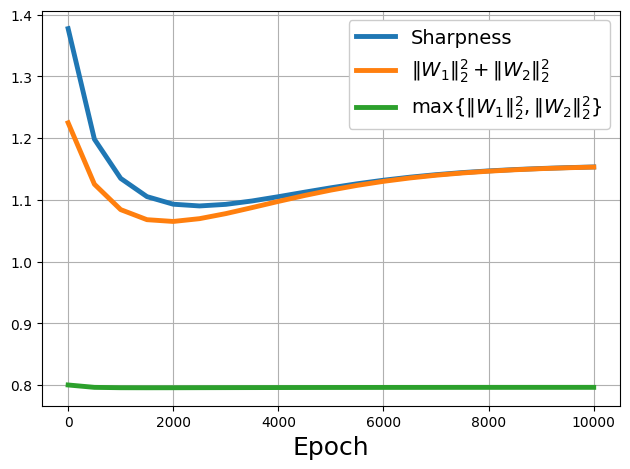}
    \includegraphics[width=0.325\linewidth]{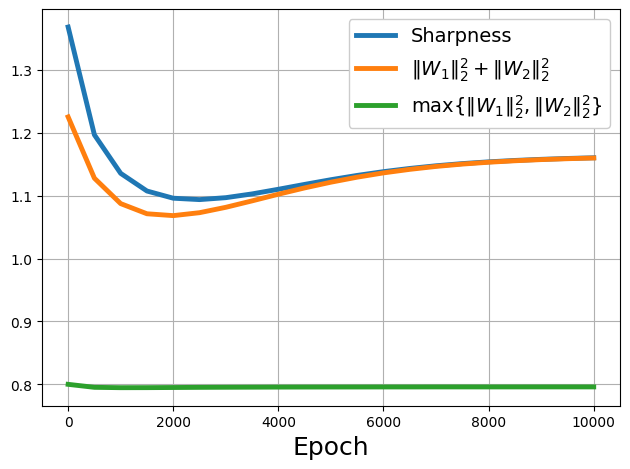}
    \includegraphics[width=0.325\linewidth]{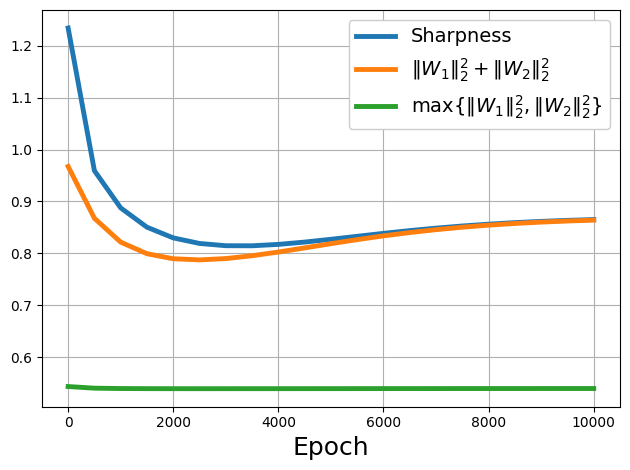}
    \caption{Top eigenvalue of the Hessian, along with the predictions in \cref{thm:shallow-networks} and \citet{singh2026cracking}, for a 2-layer MLP on (a) MNIST, (b) FashionMNIST, and (c) CIFAR10.}
\label{fig:comparing_bounds}
\end{figure}

Consider the case where $\bSigma = \sqrt{n} \identity_{d_0}$.
The spectral norm (largest eigenvalue for symmetric matrices) of the NTK is given by $\lambda_{\max}\rb{\frac{1}{n}\bJ\bJ\transpose} = \norm{\bW_1}_2^2 + \norm{\bW_2}_2^2$. This contradicts the conclusion drawn in \citet[Appendix D.3.3]{singh2026cracking}, suggesting that the maximum eigenvalue is given by $\max\cb{\norm{\bW_1}_2^2, \norm{\bW_2}_2^2}$; we believe that the discrepancy is caused by the block-diagonal structure of the Hessian assumed in their analysis. This distinction is visualized in \cref{fig:comparing_bounds}, where the sum of squared spectral-norms gets increasingly good at approximating the sharpness of the Hessian, as training progresses.

\begin{figure}[t]
\centering
    \includegraphics[width=0.45\linewidth]{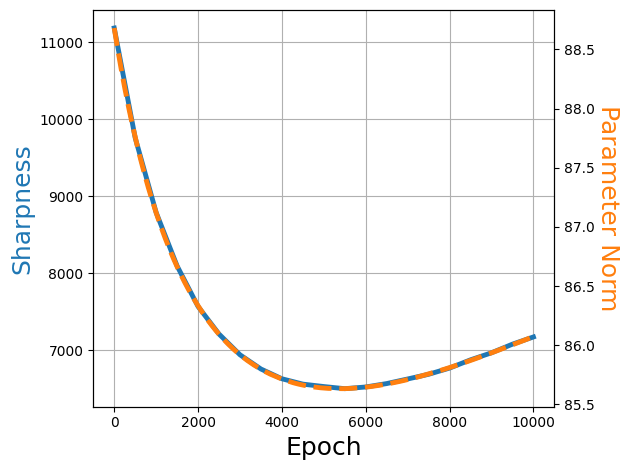}
    \includegraphics[width=0.45\linewidth]{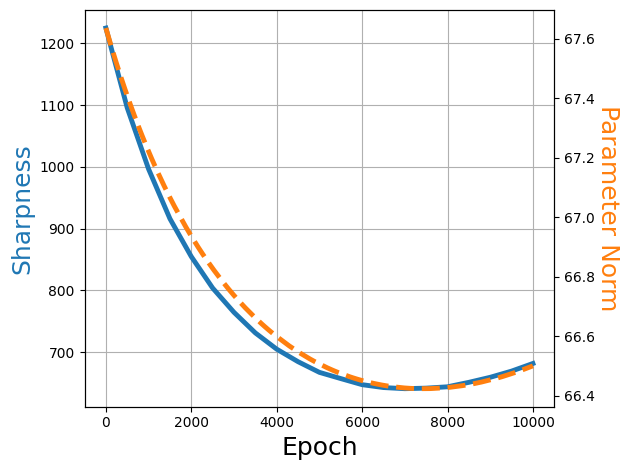}
    \caption{Sum of eigenvalues of the Hessian, computed numerically (Hutchinson’s method \cite{hutchinson89}), and squared parameter-norm of (a) a 2-layer MLP and (b) a 4-layer MLP trained on CIFAR10.}
\label{fig:sharpness_paramnorm}
\end{figure}

As for the sum of eigenvalues, $\trace\rb{\frac{1}{n}\bJ\bJ\transpose} = d_2 \norm{\bW_1}_F^2 + d_0 \norm{\bW_2}_F^2$. 
This relation provides insight into the progressive flattening and subsequent sharpening observed in early- and intermediate-training \cite{kalra2023phasediagram,kalra2025universal}, respectively, before training reaches EoS \cite{cohen2021gradient,cohen2024adaptivegradientmethodsedge,islamov2026noneuclideangradientdescentoperates} -- sharpness decreases since the parameter norm decreases in early-training, as the model output evolves to align with the labels, and both start increasing after a point \citep[Theorem 2]{singh2026unified}. In \cref{fig:sharpness_paramnorm}, we demonstrate this relation for a shallow network as well as for a deep network, showing a high amount of agreement in the dynamics of $\trace\rb{\frac{1}{n}\bJ\bJ\transpose}$ and the squared parameter-norm.

\section{Strongly Balanced Deep Networks}
\label{sec:strongly-balanced}

In this section, we derive analogous results to \cref{thm:shallow-networks} but for networks of arbitrary depth, while relying on an additional assumption on the balance between layers.

\begin{assumption}
    Assume that consecutive layers are strongly balanced, \ie $\bW_{\ell} \bW_{\ell}\transpose = \bW_{\ell+1}\transpose \bW_{\ell+1}$, $\forall \ell \in \sb{L-1}$.
\label{ass:balance}
\end{assumption}
In essence, this implies that the left singular vectors of a layer are aligned with the right singular vectors of the next layer. Moreover, all layers share their non-zero singular values; we denote them by $\sigma_i\rb{\bW}$. While arbitrary points in the parameter space need not correspond to balanced layers, initially balanced layers remain balanced under gradient flow training, provided their is no nonlinearity between them \cite{du2018algorithmic,arora18acceleration}. Even if initially unbalanced, \citet[Theorem 1]{singh2026unified} showed that the layers become increasingly balanced under weight-decay training, while \citet[Proposition 2]{ghosh25matrixfat} showed that gradient descent training at EoS \cite{cohen2021gradient} balances the singular values of the weight matrices.
Therefore, the results hereon can be seen as characterizing the local geometry of the minimizers.

\begin{theorem}
\label{thm:deep-mlp-eigs}
    Under 
    \cref{ass:balance}, the non-zero eigenvalues of the NTK are given by
    \begin{salign}
        \lambda_{ij}\rb{\frac{1}{n}\bJ\bJ\transpose} &\leq
        \frac{\norm{\bX}_2^2}{n} \cdot \begin{cases}
            \frac{\sigma_{i}^{2L}\rb{\bW} - \sigma_{j}^{2L}\rb{\bW}}{\sigma_{i}^{2}\rb{\bW} - \sigma_{j}^{2}\rb{\bW}} &\!\!\!\!, \quad \sigma_{i}\rb{\bW} \neq \sigma_{j}\rb{\bW} \\
            L \cdot \sigma^{2\rb{L-1}}\rb{\bW} &\!\!\!\!, \quad \sigma\rb{\bW} \coloneqq \sigma_{i}\rb{\bW} = \sigma_{j}\rb{\bW}
        \end{cases}
    \label{eqn:deep-mlp-eigs-b}
    \end{salign}
    with equality under \cref{ass:whitened-inputs}. 
\end{theorem}

\begin{figure}[t]
    \centering
        \includegraphics[width=0.325\linewidth]{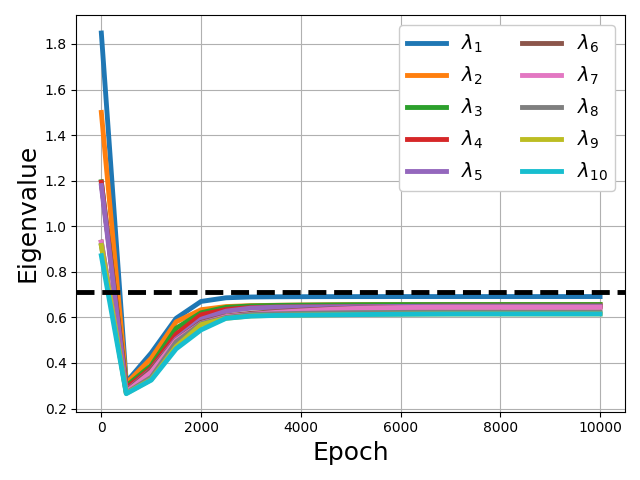}
        \includegraphics[width=0.325\linewidth]{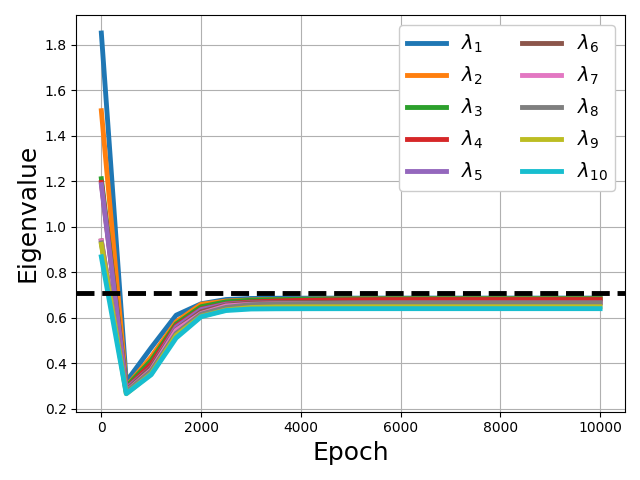}
        \includegraphics[width=0.325\linewidth]{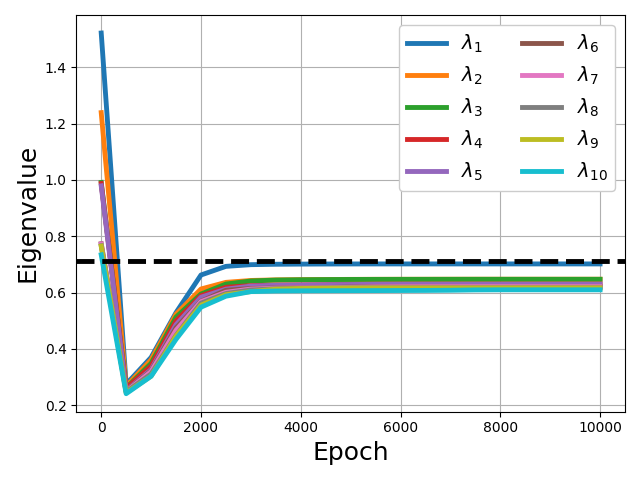}
        \caption{Top-10 eigenvalues of the Hessian and the predicted values (in black, using \cref{thm:deep-mlp-eigs}) for a 4-layer MLP on (a) MNIST, (b) FashionMNIST, and (c) CIFAR10.}
    \label{fig:deep_networks}
    \end{figure}

Hence, of the $p \coloneqq \sum_{\ell=1}^L d_{\ell}d_{\ell-1}$ eigenvalues of the Hessian, at most $\rb{\min\cb{d_0,d_1,\ldots,d_L}}^2$ are non-zero. 
This corresponds to the empirical observations suggesting that the bulk of eigenvalues of the Hessian are near-zero \cite{sagun16singularity,sagun2018empiricalanalysishessianoverparametrized,papyan2019spectrumdeepnethessiansscale}.
We visualize the case of isotropic features in \cref{fig:deep_networks}, where we note that the approximation becomes increasingly precise as training progresses.

\begin{corollary}
\label{cor:bounds}
    Under 
    \cref{ass:balance}, the spectral norm of the NTK can be bounded as
    \begin{salign}
        \frac{1}{n} \norm{\bX}_2^2 \norm{\bW}_2^{2\rb{L-1}} \leq \lambda_{\max}\rb{\frac{1}{n}\bJ\bJ\transpose} \leq
        \frac{L}{n} \norm{\bX}_2^2 \norm{\bW}_2^{2\rb{L-1}}
    \end{salign}
\end{corollary}

\begin{figure}[t]
\centering
    \includegraphics[width=0.325\linewidth]{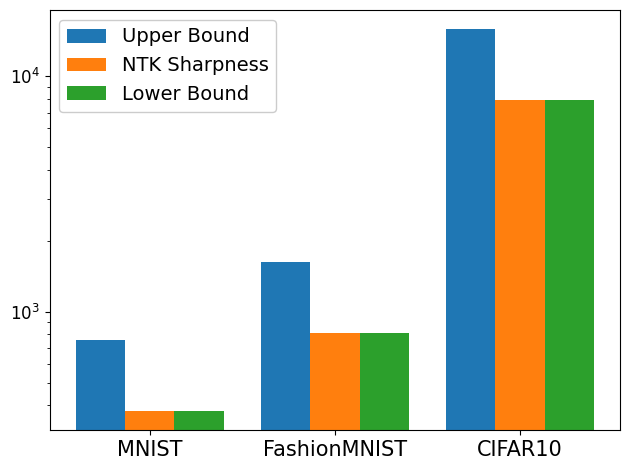}
    \includegraphics[width=0.325\linewidth]{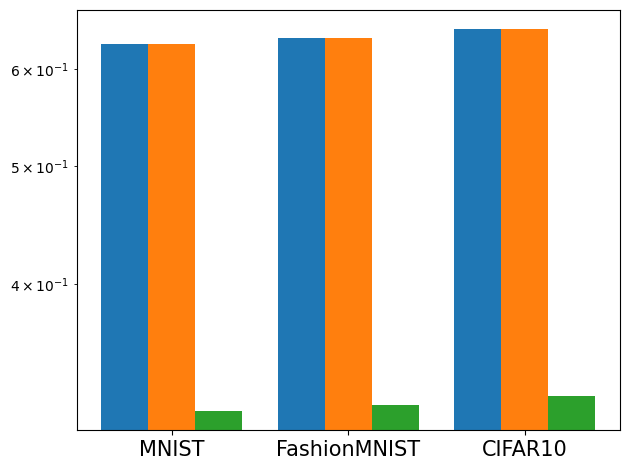}
    \includegraphics[width=0.325\linewidth]{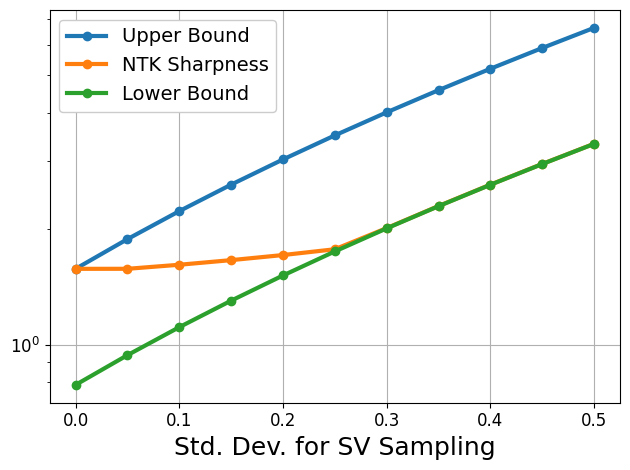}
    \caption{Sharpness of the NTK of the ordinary-least squares (OLS) solution, using (a) original datasets and (b) whitened features, along with corresponding lower and upper bounds (\cref{cor:bounds}). In (c), we redefine the feature matrix by sampling singular values from a Gaussian distribution with different scales.}
\label{fig:bounds}
\end{figure}

We distinguish two cases: 1. when the features are highly anisotropic (typical case for real-world datasets), and 2. when the features are isotropic (more idealistic). We observe that in the former case, the lower bound is tight, while in the latter case the upper bound is tight; see \cref{fig:bounds}. This suggests that the true sharpness floats between the two bounds as the degree of isotropy in the features changes.

\section{Dependence on Data}



\begin{theorem}
\label{thm:eigval-correspondence}
    Say the network implements the ordinary least-squares (OLS) solution, \ie $\bW_{1}\transpose\bW_{2}\transpose \ldots \bW_{L}\transpose = \bX\pinv\bY$, where $\cdot\pinv$ denotes the Moore-Penrose inverse. Under \cref{ass:whitened-inputs,ass:balance},
    \begin{salign}
         \sigma_i\rb{\bW} = \bSigma_{11}^{-1/L} \cdot \sigma_i^{1/L}\rb{\bU_{\bX}\transpose \bY}
    \end{salign}
    In particular, the spectral norm is given by
    \begin{salign}
        \lambda_{\max}\rb{\frac{1}{n}\bJ\bJ\transpose} = L \rb{\frac{\norm{\bX}_2^2}{n}}^{1/L} \rb{\frac{\norm{\bU_{\bX}\transpose\bY}_2^2}{n}}^{1-1/L} \leq L \rb{\frac{\norm{\bX}_2^2}{n}}^{1/L} \rb{\frac{\norm{\bY}_2^2}{n}}^{1-1/L}
    \end{salign}
    This bound is attained when $\loss\rb{\bW_{1:L};\bX,\bY} = 0$.
\end{theorem}

We visualize this correspondence in the left panel in \cref{fig:eigvals_w_balanced_init}, where we note that the theory predicts the non-zero eigenvalues exactly. 

\textbf{Dependence on dataset size.} Assuming the entries of $\bX$ and $\bY$ are constant in the datasets size, \ie $\bX_{ij}, \bY_{ij} = \bTheta_n\rb{1}$, their spectral-norms scale as $\bTheta\rb{\sqrt{n}}$, and hence, the sharpness of balanced solutions is independent of dataset size (see \cref{fig:effect_of_subset_size}), in disagreement with \citet[Figure 18]{cohen2021gradient}, which suggests that sharpness increases with $n$. Their observation may be better understood through the geometry of the loss landscape and the behavior of optimization algorithms, which are beyond the scope of this work.

\textbf{Dependence on depth.} Sharpness increases with depth when $\norm{\bX}_2 \leq \norm{\bY}_2$, or $L \geq 2\ln \frac{\norm{\bX}_2}{\norm{\bY}_2}$ -- we visualize this in the middle panel in \cref{fig:eigvals_w_balanced_init}.


\begin{figure}[t]
\centering
    \includegraphics[width=0.325\linewidth]{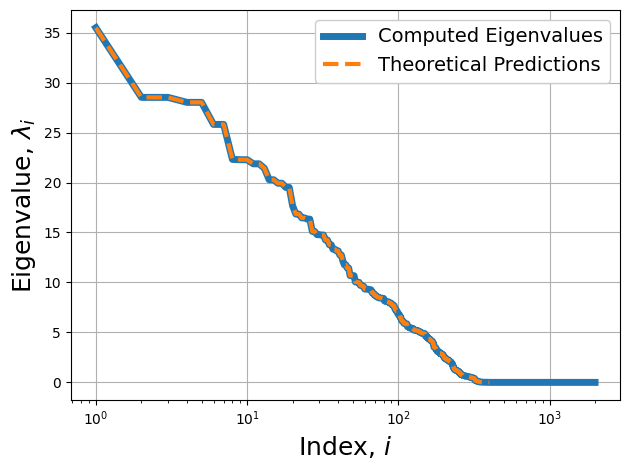}
    \includegraphics[width=0.325\linewidth]{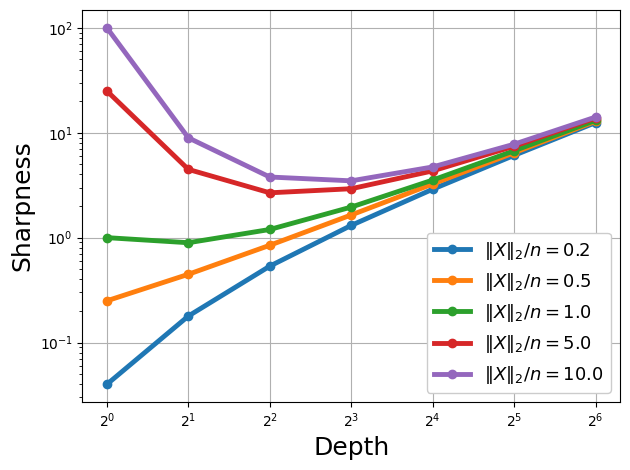}
    \includegraphics[width=0.325\linewidth]{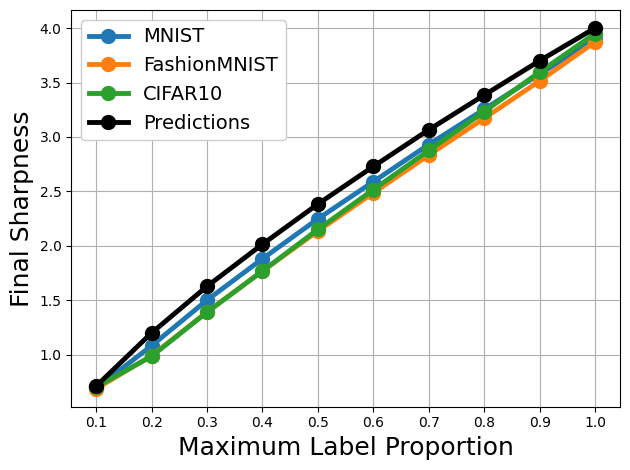}
    \caption{Computing eigenvalues of the Hessian of strongly balanced solutions using the labels, $\bY$. (a) Computed numerically and theoretically (\cref{thm:deep-mlp-eigs,thm:eigval-correspondence}) with random data. (b) On the effect of depth for different values of $\norm{\bX}_2^2/n$, with fixed $\alpha_{\max} = 0.2$. (c) Spectral norm of trained models and theoretical predictions (\cref{thm:eigval-correspondence}).}
\label{fig:eigvals_w_balanced_init}
\end{figure}

\textbf{Dependence on features.} For a centered feature matrix $\bX$, \ie the empirical mean of each feature-dimension is zero, $\{\sigma_{i}^{2}\rb{\bX}/n\}_{i=1}^{d_0}$ represents the amount of variance explained by the principle components. Under \cref{ass:whitened-inputs}, this is given by $\bSigma_{11}^2$. Hence, larger the magnitude of the features, the more spread out they are, and larger is the sharpness of the solution.

\textbf{Dependence on labels.} For one-hot encoded labels, $\{\sigma_{i}^{2}\rb{\bY}/n\}_{i=1}^{d_L}$ is the empirical distribution of labels, \ie proportion of labels from each class in the training set; in what follows, we will use $\alpha_i \coloneqq \sigma_{i}^{2}\rb{\bY}/n$ to denote these proportions.
This implies that the model learns sharper solutions when there is one class with disproportionately large number of labels; we illustrate this in the right panel in \cref{fig:eigvals_w_balanced_init}. Such datasets are, intuitively, simpler to learn, \eg as an extreme case, a dataset with \emph{all inputs belonging to one class} should be easier to learn than a dataset with uniform label distribution. This intuition is in disagreement with \citet[Caveat 2]{cohen2021gradient}, which suggests that the model achieves lower peak-sharpness on simpler datasets.



\begin{figure}[t]
\centering
    \includegraphics[width=0.325\linewidth]{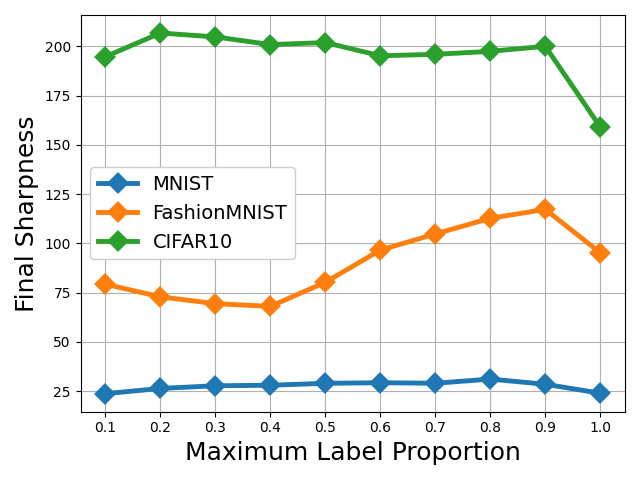}
    \includegraphics[width=0.325\linewidth]{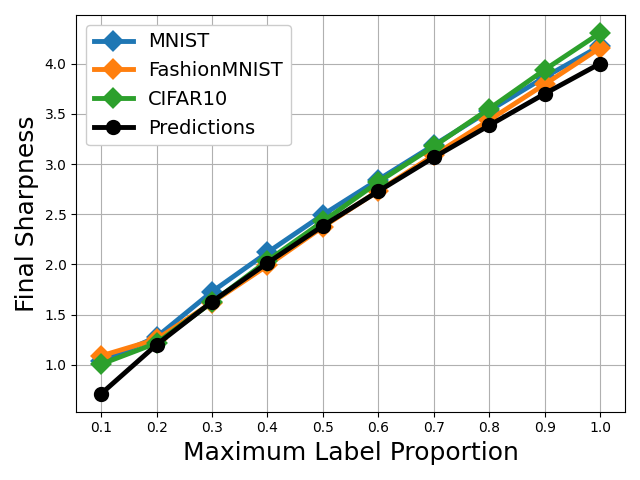}
    \includegraphics[width=0.325\linewidth]{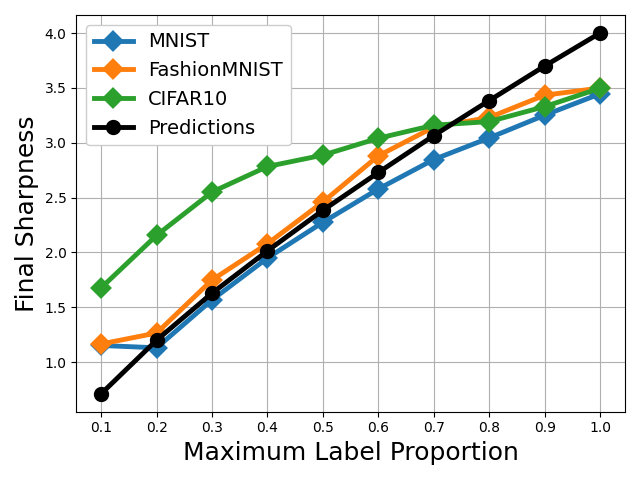}
    \caption{Ablating model assumptions: (a) raw inputs (shedding \cref{ass:whitened-inputs}), (b) Kaiming uniform initialization (shedding \cref{ass:balance}), and (c) unbalanced initialization and Tanh (nonlinear) activation.} 
\label{fig:ablations}
\end{figure}

\textbf{Ablations.} We repeat the experiment in \cref{fig:eigvals_w_balanced_init}, shedding each of 
\cref{ass:whitened-inputs,ass:balance} one-at-a-time. In the left panel in \cref{fig:ablations}, we present the results shedding the assumption on whitened inputs. In this case, our theory breaks down immediately -- the sharpness of the learnt solution turns out to have inconsistent variations with maximum label proportion in the dataset, suggesting that peak-sharpness is strongly influenced by the input-geometry. 
In the middle panel, we present the results shedding the assumption on balanced weights, and use the standard Kaiming uniform initialization \cite{he2015} implemented in the PyTorch framework. Here, we note the theoretical predictions made in previous section stay intact, as the final sharpness increasing with label imbalance. Finally, we inspect if incorporating nonlinearity breaks our theory. In the right panel, we present results for a Tanh-MLP, noting that while the results are slightly different from the theoretical predictions, they are qualitatively similar, allowing us to extend our conclusions on the effects of label distribution to such nonlinear networks.

\clearpage
\bibliography{refs}

\clearpage
\appendix

\section{Omitted Proofs}
\label{app:proofs}

\begin{proof}[\cref{thm:shallow-networks}]
    
    For $L=2$, following \cref{eqn:column-gramian}, 
    \begin{salign}
        &\bS_{1}\transpose \bS_{1} = \bW_{2} \bW_{2}\transpose \implies \lambda_{i}\rb{\bS_{1}\transpose \bS_{1}} = \sigma_i^2\rb{\bW_{2}} \\
        &\bX\bP_{1}\bP_{1}\transpose\bX\transpose = \bX \bX\transpose \implies \lambda_{i}\rb{\bX\bP_{1}\bP_{1}\transpose\bX\transpose} = \sigma_i^2\rb{\bX} \leq \bSigma_{11}^2 \\
        &\bS_{2}\transpose \bS_{2} = \identity_{d_2} \implies \lambda_{i}\rb{\bS_{2}\transpose \bS_{2}} = 1 \\
        &\bX\bP_{2}\bP_{2}\transpose\bX\transpose = \bX \bW_1\transpose \bW_1 \bX\transpose \implies \lambda_{i}\rb{\bX\bP_{2}\bP_{2}\transpose\bX\transpose} \leq \bSigma_{11}^2 \sigma_i^2\rb{\bW_{1}}
    \end{salign}
    where both upper bounds are tight under \cref{ass:whitened-inputs}. 
    Therefore,
    \begin{salign}
        \lambda_{ij}\rb{\frac{1}{n}\bJ\bJ\transpose} &\leq \frac{\norm{\bX}_2^2}{n} \cdot \rb{\sigma_{i}^2\rb{\bW_1} + \sigma_{j}^2\rb{\bW_2}}
    \end{salign}

    
    Furthermore, 
    under \cref{ass:whitened-inputs}, the two summands share their (unnormalized) eigenvectors:
    \begin{salign}
        \bq_{ij}\rb{\frac{1}{n}\bJ\bJ\transpose} &= \bu_{i} \otimes \rb{\bX \bv_{j}}
    \end{salign}
    where $\bu_{i}$ is the $i\textsuperscript{th}$ left singular vector of $\bW_2$, and $\bv_{j}$ is the $j\textsuperscript{th}$ right singular vector of $\bW_1$.
    Therefore, \emph{all} non-zero eigenvalues of the NTK 
    are given by pairwise-sums of squared singular values of the two layers:
    \begin{salign}
        \lambda_{ij}\rb{\frac{1}{n}\bJ\bJ\transpose} &= \frac{\norm{\bX}_2^2}{n}\rb{\sigma_{i}^2\rb{\bW_1} + \sigma_{j}^2\rb{\bW_2}}
    \end{salign}
    where $i\in\sb{\min\cb{d_0,d_1}}$ and $j\in\sb{\min\cb{d_1,d_2}}$. 
    
\end{proof}

\begin{proof}[\cref{thm:deep-mlp-eigs}]

    Under \cref{ass:balance}, we have $\bS_{\ell}\transpose \bS_{\ell} = \rb{\bW_{L} \bW_{L}\transpose}^{L-\ell}$ and $\bP_{\ell} \bP_{\ell}\transpose = \rb{\bW_{1}\transpose \bW_{1}}^{\ell-1}$. Therefore, following \cref{eqn:column-gramian},
    \begin{salign}
        \frac{1}{n} \bJ \bJ\transpose
        &= \frac{1}{n} \sum_{\ell=1}^{L} 
        \rb{\bW_L \bW_L\transpose}^{L-\ell}
        \otimes
        \rb{\bX \rb{\bW_1\transpose \bW_1}^{\ell-1} \bX\transpose}
    \label{eqn:stong-balance-raw}
    \end{salign}
    Using triangle inequality and \cref{ass:whitened-inputs}, the eigenvalues of the NTK are upper bounded as
    \begin{salign}
        \lambda_{ij}\rb{\frac{1}{n} \bJ \bJ\transpose}
        &\leq \frac{1}{n} \sum_{\ell=1}^{L} 
        \sigma_i^{2\rb{L-\ell}}\rb{\bW}
        \lambda_j\rb{\bX \rb{\bW_1\transpose \bW_1}^{\ell-1} \bX\transpose} \\
        &\leq \frac{\norm{\bX}_2^2}{n} \cdot \sum_{\ell=1}^{L} \sigma_{j}^{2\rb{L-\ell}}\rb{\bW} \sigma_{i}^{2\rb{\ell-1}}\rb{\bW} \label{eqn:deep-mlp-eigs-a} \\
        &= \frac{\norm{\bX}_2^2}{n} \cdot
        \begin{cases}
            \frac{\sigma_{i}^{2L}\rb{\bW} - \sigma_{j}^{2L}\rb{\bW}}{\sigma_{i}^{2}\rb{\bW} - \sigma_{j}^{2}\rb{\bW}} &\!\!\!\!, \quad \sigma_{i}\rb{\bW} \neq \sigma_{j}\rb{\bW} \\
            L \cdot \sigma^{2\rb{L-1}}\rb{\bW} &\!\!\!\!, \quad \sigma\rb{\bW} \coloneqq \sigma_{i}\rb{\bW} = \sigma_{j}\rb{\bW}
        \end{cases}
    \end{salign}

    
    If $\bSigma = \bSigma_{11}\identity_{d_0}$, then we further have that the (unnormalized) eigenvectors of each summand are given by the Kronecker products $\bq_{ij} = \bu_{i} \otimes \rb{\bX \bv_{j}}$, where $\bu_{i}$ is the $i\textsuperscript{th}$ left singular vector of $\bW_L$, and $\bv_{j}$ is the $j\textsuperscript{th}$ right singular vector of $\bW_1$. 
    Hence, the eigenvalues add up over the summands,
    and the upper bound derived above becomes tight.

\end{proof}

\begin{proof}[\cref{thm:eigval-correspondence}]
    
    The OLS solution to the problem is given by $\bW_1\transpose\bW_{2}\transpose\ldots\bW_L\transpose = \bX\pinv\bY$. Under \cref{ass:balance}, the shared singular values simplify as $\sigma_i^L\rb{\bW} = \sigma_i\rb{\bX\pinv\bY}$.
    Using \cref{ass:whitened-inputs}, we have
    \begin{salign}
        \sigma_i^L\rb{\bW} &= \bSigma_{11}\inv \cdot \sigma_i\rb{\bU_{\bX}\transpose \bY} \implies \sigma_i\rb{\bW} = \bSigma_{11}^{-1/L} \cdot \sigma_i^{1/L}\rb{\bU_{\bX}\transpose \bY}
    \end{salign}
    In particular, if $\bX\bW_1\transpose\bW_{2}\transpose\ldots\bW_L\transpose = \bY$, then
    \begin{salign}
        \bX \rb{\bW_1\transpose \bW_1}^L \bX\transpose = \bY\bY\transpose &\implies \bSigma_{11}^2 \sigma_i^{2L}\rb{\bW} = \sigma_i^{2}\rb{\bY} \\
        &\implies \sigma_i\rb{\bW} = \bSigma_{11}^{-1/L} \cdot \sigma_i^{1/L}\rb{\bY}
    \end{salign}
    All that is left is to substitute the shared singular values in \cref{thm:deep-mlp-eigs}.

\end{proof}

\section{Other Results}

\subsection{Summary of Eigenvalues}

Several works have used Hessian-based measures like the maximum eigenvalue (spectral norm for symmetric matrices) or the sum of eigenvalues (trace) to measure generalization. Intuitively, at a local minimum, the spectral norm measures the worst-case change in loss caused by a small perturbation in any direction, while the trace measures the expected change in loss caused by a random isotropic perturbation with small variance.

\begin{lemma}
\label{lem:summary-spectrum}
    The spectral norm and trace of the NTK are given by
    \begin{salign}
        &\;\lambda_{\max} \rb{\frac{1}{n}\bJ\bJ\transpose} \leq \frac{\norm{\bX}_2^2}{n} \cdot \sum_{\ell=1}^{L} \norm{\bS_{\ell}}_2^2 \norm{\bP_{\ell}}_2^2 \\
        &\;\trace\rb{\frac{1}{n}\bJ\bJ\transpose} \leq \frac{\norm{\bX}_2^2}{n} \cdot \sum_{\ell=1}^{L}
        \norm{\bS_{\ell}}_F^2 \norm{\bP_{\ell}}_F^2
    \end{salign}
    where the bound for trace is attained under \cref{ass:whitened-inputs}.
\end{lemma}

\begin{proof}
    
    Following \cref{eqn:column-gramian},
    \begin{salign}
        \frac{1}{n} \bJ \bJ\transpose
        &= \frac{1}{n} \sum_{\ell=1}^{L}
        \rb{\bS_{\ell}\transpose \bS_{\ell}} \otimes \rb{\bX \bP_{\ell} \bP_{\ell}\transpose \bX\transpose}
    \end{salign}
    Using triangle inequality,
    \begin{salign}
        \lambda_{\max} \rb{\frac{1}{n}\bJ\bJ\transpose} &\leq \frac{1}{n} \sum_{\ell=1}^{L} \lambda_{\max} \rb{\rb{\bS_{\ell}\transpose \bS_{\ell}} \otimes \rb{\bX \bP_{\ell} \bP_{\ell}\transpose \bX\transpose}} \\
        &= \frac{1}{n} \sum_{\ell=1}^{L} \norm{\bS_{\ell}}_2^2 \lambda_{\max} \rb{\bX \bP_{\ell} \bP_{\ell}\transpose \bX\transpose} \\
        &= \frac{\norm{\bX}_2^2}{n} \cdot \sum_{\ell=1}^{L} \norm{\bS_{\ell}}_2^2 \norm{\bP_{\ell}}_2^2
    \end{salign}
    
    As for the sum of eigenvalues,
    \begin{salign}
        \trace\rb{\frac{1}{n}\bJ\bJ\transpose} &= \frac{1}{n} \sum_{\ell=1}^{L} \trace \rb{\rb{\bS_{\ell}\transpose \bS_{\ell}} \otimes \rb{\bX \bP_{\ell} \bP_{\ell}\transpose \bX\transpose}} \\
        &= \frac{1}{n} \sum_{\ell=1}^{L} \norm{\bS_{\ell}}_F^2 \trace \rb{\bX \bP_{\ell} \bP_{\ell}\transpose \bX\transpose} \\
        &\leq \frac{\norm{\bX}_2^2}{n} \cdot \sum_{\ell=1}^{L}
        \norm{\bS_{\ell}}_F^2 \norm{\bP_{\ell}}_F^2
    \end{salign}
\end{proof}

\subsection{Other Figures}

\begin{figure}[h]
\centering
    \includegraphics[width=0.325\linewidth]{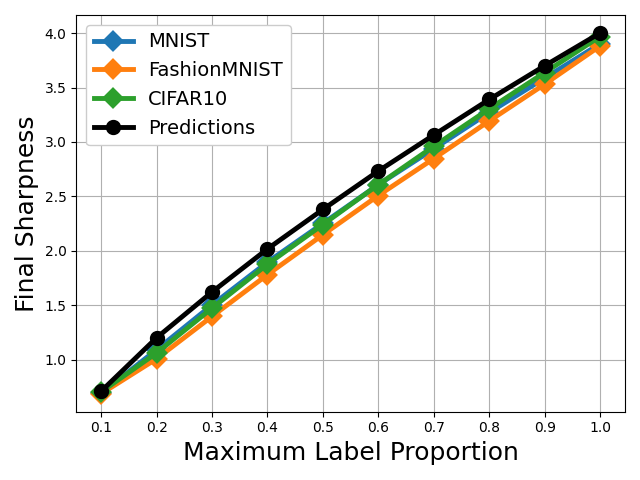}
    \includegraphics[width=0.325\linewidth]{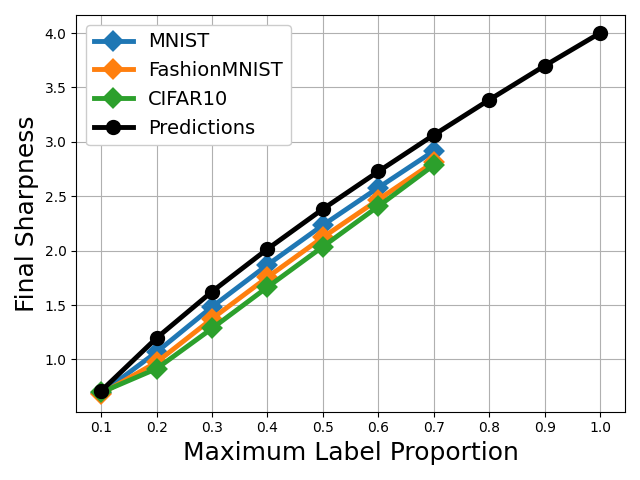}
    \includegraphics[width=0.325\linewidth]{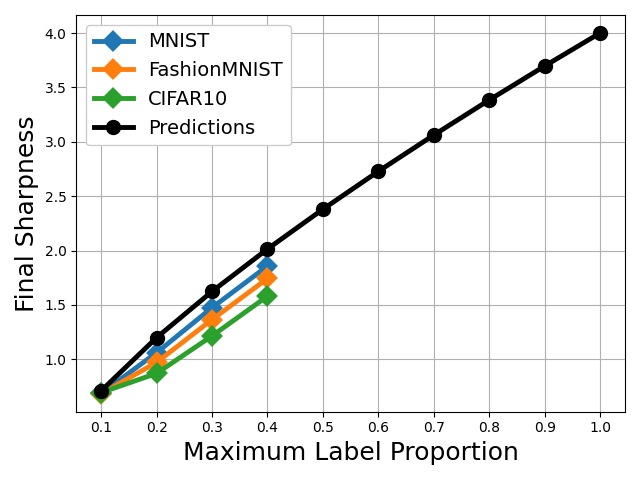}
    \caption{Effect of dataset size on spectral norm of trained networks: (a) $n=4$k, (b) $7$k and (c) $10$k. Theoretical predictions correspond to \cref{thm:deep-mlp-eigs,thm:eigval-correspondence}. Some high $\alpha_{\max}$ are not presented for $n=7$k and $10$k since they requires more samples with label $0$ than available -- $5$k in the datasets considered.}
\label{fig:effect_of_subset_size}
\end{figure}

\section{Experimental Details}

\textbf{Data Whitening.} Given the SVD of a feature matrix $\real^{n\times d_0} \ni \bX' = \bU_{\bX}\bSigma\bV_{\bX}\transpose$, we redefine the features to be proportional to the polar factor, \ie $\bX = \sqrt{n}\sigma_{\bx}\bU_{\bX}\bV_{\bX}\transpose$. 

\textbf{Sampling Balanced Weights.} Assume we are given a sequence of $k=\min\cb{d_0,d_1,\ldots,d_L}$ singular values; we arrange them to create a diagonal matrix, $\bS \in \real^{k\times k}$. To construct a sequence of balanced weight matrices, we start by sampling the right singular vectors of $\bW_1$ by orthogonalizing a random $d_0\times k$ matrix with $\cN\rb{0,1}$ entries; denote it by $\bV_1 \in \real^{d_0\times k}$. Then, for each $\ell \in \sb{L}$, do
\begin{enumerate}[itemsep=-2mm]
    \item Sample the left singular vectors of $\bW_{\ell}$ by orthogonalizing a random $d_{\ell}\times k$ matrix with $\cN\rb{0,1}$ entries; denote it by $\bU_{\ell} \in \real^{d_1\times k}$.
    \item Define the weight matrix as $\bW_{\ell} \coloneqq \bU_{\ell}\bS\bV_{\ell}\transpose \in \real^{d_{\ell}\times d_{\ell-1}}$.
    \item Set the right singular vectors of the next layer as the left singular vectors of the previous layers, $\bV_{\ell+1} \coloneqq \bU_{\ell}$.
\end{enumerate}

\textbf{Subsampling.} With real-world datasets MNIST \cite{lecun2010mnist}, FashionMNIST \cite{xiao2017fashionmnist} and CIFAR10 \cite{Krizhevsky09learningmultiple}, we use subsets of $n=5,000$ samples. If the maximum label proportion $\alpha_{\max}$ is set, we take $\lfloor(1-\alpha_{\max})n\rfloor$ samples from each of the classes $1\!-\!9$, where $\lfloor\cdot\rfloor$ denotes the floor function, and the rest of the samples from class $0$.

\textbf{Training.} All models are trained using full-batch gradient descent for $10,000$ epochs, with an initial learning rate $\eta_{\max}$ which is successively halved anytime $<0.1\%$ improvement in loss is observed over $100$ epochs.

\textbf{\cref{fig:shallow_networks}:} Top-10 eigenvalues of the Hessian, computed numerically using LOBPCG and analytically using \cref{thm:shallow-networks}. Datasets are orthogonalized with $\sigma_{\bx}=1$, and our model is a 2-layer Linear-MLP with width 256, trained with $\eta_{\max}=1e-3$.

\textbf{\cref{fig:comparing_bounds}:} Top eigenvalue of the Hessian, along with the predictions in \cref{thm:shallow-networks} and \citet{singh2026cracking}. Datasets are orthogonalized with $\sigma_{\bx}=1$, and our model is a 2-layer Linear-MLP with width 256, trained with $\eta_{\max}=1e-3$.

\textbf{\cref{fig:sharpness_paramnorm}:} Sum of eigenvalues of the Hessian is computed numerically using Hutchinson’s method \cite{hutchinson89}. Datasets are orthogonalized with $\sigma_{\bx}=1$, and our models are a 2-layer Linear-MLP with width 256 and a 4-layer Linear-MLP with width 64, trained with $\eta_{\max}=5e-4$.

\textbf{\cref{fig:deep_networks}:} Top-10 eigenvalues of the Hessian, computed numerically using LOBPCG and analytically using \cref{thm:deep-mlp-eigs}. Datasets are orthogonalized with $\sigma_{\bx}=1$, and our model is a 4-layer Linear-MLP with width 64, trained with $\eta_{\max}=1e-2$.

\textbf{\cref{fig:bounds}:} The OLS solution is computed as $\bW\transpose \coloneqq \bW_1\transpose\bW_2\transpose\ldots\bW_L\transpose = \bX\pinv \bY$, and a network with balanced layers is constructed to simulate this solution by setting $\bV_1$ and $\bU_L$ as the right and left singular vectors of $\bW$. Our model is a 2-layer Linear-MLP with width 128.

\textbf{\cref{fig:eigvals_w_balanced_init}:}
\begin{itemize}[itemsep=-2mm]
    \item \textbf{Left:} We sample $\bX' \in \real^{20\times 100}$ with standard normal entries, and then orthogonalize it with $\sigma_{\bx} = 3$. We sample $d=20$ singular values from $\cU\sb{0,1}$, and use it sample a sequence of $L=5$ balanced weight matrices, with $\bW_{\ell} \in \real^{d\times d}$. Finally, set $\bY = \bW_5\bW_4\ldots\bW_1\bX$ so that 
    the GGN approximation is exact.
    \item \textbf{Right:} Spectral norm of the Hessian is computed numerically using LOPPCG \cite{Knyazev01lobpcg,stathopoulos02lobpcg} and analytically using \cref{thm:deep-mlp-eigs,thm:eigval-correspondence}. Datasets are orthogonalized with $\sigma_{\bx}=1$, and our model is a 4-layer Linear-MLP with width 64, initialized with balanced weights, and trained with $\eta_{\max}=1e-2$.
\end{itemize}

\textbf{\cref{fig:ablations}:} Spectral norm of the Hessian is computed numerically using LOBPCG and analytically using \cref{thm:deep-mlp-eigs,thm:eigval-correspondence}.
\begin{itemize}[itemsep=-2mm]
    \item \textbf{Left:} Datasets are \emph{not} orthogonalized, and our model is a 4-layer Linear-MLP with width 64, initialized with balanced weights.
    \item \textbf{Middle:} Datasets are orthogonalized with $\sigma_{\bx}=1$, and our model is a 4-layer Linear-MLP with width 64, initialized with Kaiming uniform initialization \cite{he2015} (standard in PyTorch).
    \item \textbf{Right:} Datasets are orthogonalized with $\sigma_{\bx}=1$, and our model is a 4-layer Tanh-MLP with width 64, initialized with balanced weights.
\end{itemize}
All models are trained with $\eta_{\max}=1e-2$.


\textbf{\cref{fig:effect_of_subset_size}:} Spectral norm of the Hessian, computed numerically using LOPPCG and analytically using \cref{thm:deep-mlp-eigs,thm:eigval-correspondence}. Datasets are orthogonalized with $\sigma_{\bx}=1$, and our model is a 4-layer Linear-MLP with width 64, trained with $\eta_{\max}=1e-2$.

\end{document}